\documentclass[12pt]{article}

\usepackage[english]{babel}
\usepackage{csquotes}

\usepackage[margin=1in]{geometry}

\usepackage{amsmath}
\usepackage{amsthm}
\usepackage{graphicx}
\usepackage[colorlinks=false]{hyperref} 
\usepackage{enumitem}
\setenumerate[0]{label=(\alph*)}
\usepackage{amsfonts}
\usepackage{times}
\usepackage{lscape}

\newtheorem{theorem}{Theorem}[section]

\newtheorem{prop}[theorem]{Proposition}
\newtheorem{lemma}[theorem]{Lemma}  

\theoremstyle{definition}

\newtheorem{rmk}{Remark}

\theoremstyle{remark}

\newcommand{\bvec}[1]{\boldsymbol{#1}}
\DeclareMathOperator{\mcc}{MCC}
\DeclareMathOperator{\wmcc}{WMCC}
\DeclareMathOperator{\mpc}{MPC}
\DeclareMathOperator{\wmpc}{WMPC}
\DeclareMathOperator{\ecc}{ECC}
\DeclareMathOperator{\wecc}{WECC}
\DeclareMathOperator{\tr}{tr}

\newcommand{\nsp}{\hspace*{-0.02in}}
\usepackage{authblk}

\usepackage{xcolor}
\definecolor{addblue}{rgb}{0.1,0,0.8}

\usepackage[normalem]{ulem}

\definecolor{darkgrn}{rgb}{0,0.75,0}

\date{}

\title{Weighted MCC: A Robust Measure of Multiclass Classifier Performance for Observations\\ with Individual Weights}

\author[1]{Rommel Cortez\thanks{Corresponding author; rommel.cortez@wsu.edu}}
\author[1]{Bala Krishnamoorthy}
\affil[1]{Department of Mathematics and Statistics, Washington State University}

\begin{document}
\maketitle

\begin{abstract}
  Several performance measures are used to evaluate binary and multiclass classification tasks.
  But individual observations may often have distinct weights, and none of these measures are sensitive to such varying weights.
  We propose a new weighted Pearson-Matthews Correlation Coefficient (MCC) for binary classification as well as weighted versions of related multiclass measures.
  The weighted MCC varies between $-1$ and $1$.
  But crucially, the weighted MCC values are higher for classifiers that perform better on highly weighted observations, and hence is able to distinguish them from classifiers that have a similar overall performance and ones that perform better on the lowly weighted observations.
  Furthermore, we prove that the weighted measures are \emph{robust} with respect to the choice of weights in a precise manner:
  if the weights are changed by at most $\epsilon$, the value of the weighted measure changes at most by a factor of $\epsilon$ in the binary case
  and by a factor of $\epsilon^2$ in the multiclass case.
  Our computations demonstrate that the weighted measures clearly identify classifiers that perform better on higher weighted observations, while the unweighted measures remain completely indifferent to the choices of weights.
\end{abstract}

\section{Introduction}

Multiple measures are used to evaluate the performances of binary and multiclass classifiers \cite{Th2021,LaSo2009}.
Accuracy, precision, recall, F1-score, area under ROC (AUC), and the Pearson-Matthews Correlation Coefficient (MCC) are among the commonly used measures.
Compared to most other measures, the MCC has several properties that make it simultaneously more informative and easier to use for evaluating both binary and multiclass classification performances.
The MCC is defined using all parts of the confusion matrix to compute a value between $1$ and $-1$.
But most, if not all, of these measures do not distinguish individual observations based on prescribed weights.

Variable weights have been considered at the level of classes, e.g., when working with imbalanced datasets \cite{BlLu2013}.
But in several applications, observations within each class could have distinct weights.
Categorization of plots of land from aerial images in Geographic Object-Based Image Analysis (GEOBIA) \cite{RadouxBogaert} presents such an instance.
The objects are represented by polygons and a standard task is to label them into one of several classes (farmland, forest, urban areas, etc.).
Thematic accuracy is size based, leading to the incorporation of weights corresponding to areas of the polygons.

Another typical application is in real-time semantic segmentation, i.e., pixel classification in autonomous driving \cite{PaChKiCu2016}.
While distinguishing road pixels from building or stop sign pixels is important, getting this assignment correct for the road pixels that lie close to the sidewalk or to a vehicle parked on the side of the road is more critical than for a pixel in the middle of the road.
But as far as we are aware, none of the classification performance measures can incorporate distinct weights for individual observations.

\subsection{Our contributions}
\label{contributions}

We define weighted versions of the Pearson-Matthews Correlation Coefficient (MCC) for binary classification (Sections \ref{bin_class}) and related weighted measures of performance for multiclass classification (Section \ref{mult_class}).
These measures incorporate distinct weights for each observation in the dataset, and yield higher vales for predictions that perform better on the more important, i.e., highly weighted, instances.
We study the mathematical properties of these weighted measures in terms of how robust they are with respect to small changes in the chosen weights.
We show that when the weights are changed by a small amount of at most $\epsilon$, the MCC for binary classification changes at most by a factor of $\epsilon$ (Theorem \ref{thm:mccstable}).
Correspondingly, we show that various measures of multiclass classification change by at most factors of $\epsilon^2$ (Theorem \ref{thm:msrstable})

We present computational experiments (Section \ref{computations}) showing the effect of varying weights of individual observations on measures of performance evaluation.
Our weighted measures clearly rank as higher those predictions that perform better on the more important, i.e., highly weighted, observations in both the binary and multiclass cases.
On the other hand, the standard (unweighted) measures fail to distinguish such good predictions from those that perform worse on the highly weighted observations.

\subsection{Related work}
\label{rel_work}

Several overviews of classification methods are available \cite{Th2021,Reetaletal2024}.
Reinke et al.~\cite{Reetaletal2024} focused on the pitfalls of classification methods in order to assist researchers in the selection of a measure appropriate for their goals.
Measures are often formulated in terms of the confusion matrix corresponding to the classification.
Several authors comment on the effect of imbalanced confusion matrices \cite{BlLu2013}.
In this regard, the $\mcc$ can perform better than other metrics, such as $F_1$ or accuracy, for binary classification, in part because $\mcc$ relies on the performance in both classes and the latter two are focused on the positive class.
Robustness of some measures with respect to their confusion matrix was investigated by Sokolova and Lapalme \cite{LaSo2009}.
Attention was placed on the invariance properties of the measure with respect to changes in the confusion matrix.  

A motivating work for us was that of Stoica and Babu \cite{StBa2024}, who considered the definition of $\mcc$ and multiclass measures in a systematic manner.
While these authors mentioned briefly the option of varying the weights for each class, a detailed analysis of their effect was not presented.
More importantly, distinct weights for individual observations were not considered by any other studies.

An analysis of various performance measures and their applicability is available in \cite{LaSo2009}.  A general discussion on  classification methods is also available in \cite{Th2021}.

A reviewer of a previous version of this manuscript made us aware of a recent implementation of weighted classification performance assessment measures presented as part of Scikit-Learn \cite{SeJe2025}.
But as far as we can see, such definitions have yet to be presented formally in the literature.
Furthermore, no theoretical guarantees of the robustness of such measures to small changes in weights are known.

\section{Definitions}
    \label{definitions}

    We define weighted performance measures using settings similar to that of Stoica and Babu \cite{StBa2024}.
    Our definitions naturally include the default (unweighted) measures as special cases.
    For brevity, we refer to the weighted measures by their default names, e.g., MCC.

    \subsection{Binary case}
    We define 
    \begin{equation}
        \label{mcc_def}
        \mcc =
        \frac{TP \cdot TN-FP\cdot FN}
             {\sqrt{(TP \nsp+\nsp FP)(TP \nsp+\nsp FN)(TN \nsp+\nsp FP)(TN \nsp+\nsp FN)}}
    \end{equation}
    where $TP, TN, FP, FN$ represent numbers of true positives, true negatives, false positives, and false negatives. 
    With $N$ samples (or observations), we associate two binary $N$-vectors (or sequences) $\bvec{t}$ and $\bvec{c}$ representing the true values and predictions, respectively.
    Let $S$ be a positive definite $N$-matrix of positive weights.
    In the default, unweighted setting, $S=I$, the identity matrix.
    Then the values used for $\mcc$ are calculated as
    \begin{align*}
        TP &= \bvec{t}^T S\bvec{c} = \langle\bvec{t},\bvec{c}\rangle, \\
        TN &= (\bvec{1-t})^T S(\bvec{1-c}) = \langle\bvec{1-t},\bvec{1-c}\rangle, \\
        FP &= (\bvec{1-t})^T S\bvec{c} = \langle\bvec{1-t},\bvec{c}\rangle, \text{ and }\\
        FN &= \bvec{t}^T S(\bvec{1-c}) = \langle\bvec{t},\bvec{1-c}\rangle.
    \end{align*}

    \vspace*{-0.1in}
    $\bvec{1}$ is the $N$-vector of ones.
    When needed, we write $\langle\bvec{t},\bvec{c}\rangle_S$ to emphasize the dependence on weight matrix $S$.
    While $S$ could be any positive definite matrix in general,
    we take the \textit{weight matrix} to be a positive definite symmetric matrix, and focus in particular on the case in which $S$ is diagonal.
    
    The expression for $\mcc = \mcc(\bvec{t},\bvec{c})$ is given as
    \begin{align}
        \hspace*{-0.1in} \mcc & =   \frac{\langle\bvec{t},\bvec{c}\rangle\langle\bvec{1}-\bvec{t},\bvec{1}-\bvec{c}\rangle - \langle\bvec{1}-\bvec{t},\bvec{c}\rangle \langle\bvec{t},\bvec{1}-\bvec{c}\rangle}{\sqrt{\langle\bvec{t},\bvec{1}\rangle \langle\bvec{1},\bvec{c}\rangle \langle\bvec{1-t},\bvec{1}\rangle \langle\bvec{1},\bvec{1-c}\rangle}} \label{mcc1}\\[1em]
        & =  \frac{\langle\bvec{t},\bvec{c}\rangle\langle\bvec{1},\bvec{1}\rangle - \langle\bvec{t},\bvec{1}\rangle \langle\bvec{1},\bvec{c}\rangle}{\sqrt{\langle\bvec{t},\bvec{1}\rangle \langle\bvec{1},\bvec{c}\rangle \langle\bvec{1-t},\bvec{1}\rangle \langle\bvec{1},\bvec{1-c}\rangle}} \label{mcc2}.
    \end{align}

    \subsection{Multiclass case}
    With $K \geq 3$ classes, two binary sequences are defined for each class $k$:
    \[\{t_n(k)\}_{n=1}^N\qquad\text{and}\qquad \{c_n(k)\}_{n=1}^N,\]
    where the true value $t_n(k)=1$ if the $n$th sample is in the $k$th class, and $0$ otherwise.
    The predictions $c_n(k)$ are defined similarly.
    The sequences can be viewed as column vectors
    \[\bvec{t}_n = \begin{bmatrix} t_n(1)\\ \vdots \\t_n(K)\end{bmatrix}\qquad\text{and}\qquad \bvec{c}_n = \begin{bmatrix} c_n(1) \\ \vdots \\ c_n(K)\end{bmatrix}.\]
    These vectors are used to form the covariance matrix
    \begin{equation}
        \label{covmat}
        R_{tc} = \frac{1}{\sum_{i=1}^N S_i} \, \sum_{n=1}^N S_n \bigl( \bvec{t}_n-\bar{\bvec{t}}\bigr)\bigl(\bvec{c}_n-\bar{\bvec{c}}\bigr)^{\top}\,,
    \end{equation}
    where $S_n$'s represent weights and 
    \[\label{tbar}\bar{\bvec{t}} = \frac{1}{\sum_{i=1}^N S_i}\sum_{n=1}^N S_n\bvec{t}_n\,,\qquad 
    \bar{\bvec{c}} = \frac{1}{\sum_{i=1}^N S_i}\sum_{n=1}^N S_n\bvec{c}_n.\]
    $R_{tc}$ is a square, $K\times K$, matrix.
    By taking $\bvec{t} = [\bvec{t}_1 \,\bvec{t}_2\,\ldots\,\bvec{t}_N]$, the matrix with column vectors $\bvec{t}_n$, and $\bvec{c} = [\bvec{c}_1 \,\bvec{c}_2\,\ldots\,\bvec{c}_N]$, the above can be represented as
    \begin{eqnarray}
        \label{eq:tbar}
        \bar{\bvec{t}} &=& \frac{1}{\tr(S)}\,\bvec{t}S\bvec{1}, \\
        \bar{\bvec{c}} &=& \frac{1}{\tr(S)}\,\bvec{c}S\bvec{1},\text{ and } \label{eq:cbar} \\
        R_{tc} & = & \frac{1}{\tr(S)}\,(\bvec{t}-\bar{\bvec{t}}\bvec{1}^{\top})S(\bvec{c}-\bar{\bvec{c}}\bvec{1}^{\top})^{\top},         \label{eq:Rtc}
    \end{eqnarray}
    where $S$ is a diagonal matrix containing the weights.
    %
    We consider the following metrics, based on the multivariate Pearson correlation ($\mpc$) coefficients:
    \begin{description}
        \item[Extended Correlation Coefficient ($ECC$):]
        \begin{equation}
            \label{ecc}
            R_K = \frac{\tr(R_{tc})}{\bigl[ \tr(R_{tt})\,\tr(R_{cc})\bigr]^{1/2}}  \hspace*{0.65in}
        \end{equation}
    \end{description}

        \begin{equation}
            \label{mpc1}
            \mpc_1 = \frac{\sum_{k=1}^{K}\bigl[R_{tc}\bigr]_{kk}}{\sum_{k=1}^{K}\bigl(\bigl[R_{tt}\bigr]_{kk}\bigl[R_{cc}\bigr]_{kk}\bigr)^{1/2}} \hspace*{0.25in}
        \end{equation}

        \begin{equation}
            \label{mpc2}
            \mpc_2 = \frac{1}{K}\sum_{k=1}^{K} \frac{\bigl[R_{tc}\bigr]_{kk}}{\bigl(\bigl[R_{tt}\bigr]_{kk}\bigl[R_{cc}\bigr]_{kk}\bigr)^{1/2}}
        \end{equation}
        We use a superscript of $S$, i.e., $R_K^S$, $\mpc_1^S$, and $\mpc_2^S$, to emphasize the dependence on weights  $S=\{S_i\}$.

    \begin{rmk}
        \label{rmk_0den}
        It is possible that terms with $0$ denominator occur in $\mpc_2$ (Eqn.~\ref{mpc2}).
        From the definition of covariance matrices in Eqn.~(\ref{covmat}), one can see that a $0$ in the denominator occurs only if one also appears in the numerator.
        While it may happen that the limit of the quotient becomes arbitrary or does not exist,
        we use a continuity argument \cite{ChJu2020} to set such terms to $0$ in most cases.
    \end{rmk}

\section{Weighted MCC for Binary Classes}
    \label{bin_class}
    
    We present several properties of the $\mcc$.
    The first one is known for the unweighted case, but our proof technique works for the weighted case in a unified manner.

    \begin{prop}
        \label{p1}
        $-1\leq \mcc \leq 1$.
    \end{prop}
    
    \begin{proof}
      Let $\bvec{a}$ and $\bvec{b}$ be binary vectors.
      We get
      \[0\leq \langle\bvec{a},\bvec{b}\rangle \leq \min\big\{\langle\bvec{a},\bvec{a}\rangle, \langle\bvec{b},\bvec{b}\rangle\big\} ~\text{ and }
      \langle\bvec{a},\bvec{a}\rangle = \langle\bvec{a},\bvec{1}\rangle.\]
        Also, the Cauchy-Schwarz inequality implies
        \[0\leq \frac{\langle\bvec{a},\bvec{b}\rangle}{\sqrt{\langle\bvec{a},\bvec{1}\rangle \langle\bvec{b},\bvec{1}\rangle}} \leq 1. \quad \text{ Then }\]
        \begin{equation*}
          0 \leq  \quad \frac{\langle\bvec{t},\bvec{c}\rangle\langle\bvec{1}-\bvec{t},\bvec{1}-\bvec{c}\rangle}{\sqrt{\langle\bvec{t},\bvec{1}\rangle \langle\bvec{1},\bvec{c}\rangle \langle\bvec{1-t},\bvec{1}\rangle \langle\bvec{1},\bvec{1-c}\rangle}}\, , 
          \, \frac{\langle\bvec{1}-\bvec{t},\bvec{c}\rangle \langle\bvec{t},\bvec{1}-\bvec{c}\rangle}{\sqrt{\langle\bvec{t},\bvec{1}\rangle \langle\bvec{1},\bvec{c}\rangle \langle\bvec{1-t},\bvec{1}\rangle \langle\bvec{1},\bvec{1-c}\rangle}} \leq 1
        \end{equation*}
        so that the difference lies between $-1$ and $1$. 
    \end{proof}
    
    We study robustness of $\mcc$ with respect to the weights.

    \begin{prop}
        \label{p2}
        $\mcc$ is continuous with respect to the weight matrix $S$.
    \end{prop}
    \begin{proof}
        Let $S$ and $W$ be weight matrices such that $\max{|S_{ii}-W_{ii}|}<\varepsilon$.
        For binary $N$-vectors $\bvec{t}$ and $\bvec{c}$,
        \[\bigl|\langle\bvec{t},\bvec{c}\rangle_{S-W}\bigr| < N\varepsilon,\]
        by the triangle inequality.
        This means that $\langle\bvec{t},\bvec{c}\rangle_S$ is continuous in $S$, for any choice of $\bvec{t}$ and $\bvec{c}$.
        Then the continuity of the operations involved in the calculation of $\mcc$ implies that $\mcc$ is continuous with respect to $S$.  
    \end{proof}

    The next lemma, which could be of independent interest, gives bounds on arithmetic operations that we use in proving the main results.

    \begin{lemma}
        \label{lem:bdops}
        Let $0<\epsilon<1$.
        \begin{enumerate}
           \item Let $x_i$$>$$0, M$$=$$\max\{x_i\}, m$$=$$\min\{x_i\}$, and $\epsilon<\frac{m}{2}$.  Then 
             \begin{align*}
               \prod x_i\, - \, \epsilon2^n\max_{k=0}^n\{M^k\} < \prod(x_i-\epsilon) 
               < \prod (x_i+\epsilon)  < \prod x_i + \epsilon2^n\max_{k=0}^n\{M^k\}.
             \end{align*}

            \item For $x>0$ and $\epsilon < \frac{x}{2}$,
            \[\frac{4\epsilon}{9x^2} < \left|\frac{1}{x}- \frac{1}{x\pm \epsilon}\right| < \frac{4\epsilon}{x^2}.\]

            \item For $x>0$ and $\epsilon < \frac{x}{2}$,
            \[\frac{\epsilon}{\sqrt{6x}} < \bigl| \sqrt{x}-\sqrt{x\pm\epsilon} \bigr| < \frac{\epsilon}{\sqrt{2x}}.\]

            \item For $x>0$ and $\epsilon< \frac{x}{2}$,
            \[\frac{\sqrt{2}\,\epsilon}{3x\sqrt{3x}} < \left| \frac{1}{\sqrt{x}} - \frac{1}{\sqrt{x\pm\epsilon}} \right| < \frac{\sqrt{2}\,\epsilon}{x\sqrt{x}}.\]

            \item For $x_1,x_2,y_1,y_2>0$ with $\vert x_1-x_2\vert < \epsilon < x_1/2$ and $\vert y_1-y_2\vert < \delta < y_1/2$, 
            \[\left\vert \frac{x_1}{y_1}-\frac{x_2}{y_2}\right\vert < \frac{x_1}{y_1}\left[\frac{4\delta}{y_1}+\epsilon\left(1+\frac{4\delta}{y_1}\right)\right].\]
        \end{enumerate}        
    \end{lemma}
    
    \begin{proof}
        \begin{enumerate}
            \item Expanding the product,
            \begin{align*}
                \prod(x_i+\epsilon) &\leq \prod x_i + \sum_{k=1}^n \binom{n}{k}\epsilon^kM^{n-k} \\
                &< \prod x_i + \epsilon(1+1)^n\max_{k=0}^n\{M^k\}\,,
            \end{align*}
            and, since $\epsilon<m/2$,
            \begin{align*}
                \prod(x_i+\epsilon) &> \prod x_i - \sum_{k=1}^n \binom{n}{k}\epsilon^kM^{n-k} \\
                &> \prod x_i - \epsilon(1+1)^n\max_{k=0}^n\{M^k\}\,.
            \end{align*}

            \item The claim follows from 
            \[\frac{\epsilon}{\left(\frac{3}{2}x\right)^2} < \frac{1}{x}-\frac{1}{x+\epsilon} < \frac{\epsilon}{x^2}\,,~~\text{ and }\]
            \[\frac{\epsilon}{x^2} < \frac{1}{x-\epsilon} - \frac{1}{\epsilon} = \frac{\epsilon}{(x-\epsilon)x} < \frac{\epsilon}{\left(\frac{1}{2}x\right)^2}\,.\]

            \item The claim follows from 
            \[\frac{\epsilon}{\sqrt{6x}} < \sqrt{x+\epsilon}-\sqrt{x} = \frac{\epsilon}{\sqrt{x+\epsilon} + \sqrt{x}} < \frac{\epsilon}{2\sqrt{x}},\]
            and 
            \[\frac{\epsilon}{2\sqrt{x}} < \sqrt{x}-\sqrt{x-\epsilon} = \frac{\epsilon}{\sqrt{x} + \sqrt{x-\epsilon}} < \frac{\epsilon}{\sqrt{2x}}.\]

            \item The claim follows from 
            \begin{align*}
                \frac{\epsilon}{3\sqrt{\frac{3}{2}}\,x\sqrt{x}} &< \frac{1}{\sqrt{x}} - \frac{1}{\sqrt{x+\epsilon}} \\
                &= \frac{\epsilon}{\sqrt{x}\sqrt{x+\epsilon}\bigl(\sqrt{x+\epsilon}+\sqrt{x}\bigr)} \\
                &< \frac{\epsilon}{x(2\sqrt{x})}\,,
            \end{align*}
            \begin{align*}
                \text{and }~~\frac{\epsilon}{x(2\sqrt{x})} &< \frac{1}{\sqrt{x-\epsilon}}-\frac{1}{\sqrt{x}} \\ &= \frac{\epsilon}{\sqrt{x}\sqrt{x+\epsilon}\bigl(\sqrt{x}+\sqrt{x-\epsilon}\bigr)} \\
                &< \frac{\epsilon}{\frac{1}{2}x\bigl(2\sqrt{\frac{1}{2}x}\bigr)}\,.
            \end{align*}

            \item Applying the triangle inequality and using Part (b):
            \begin{align*}
                \left\vert\frac{x_1}{y_1}-\frac{x_2}{y_2}\right\vert &\leq x_1\left\vert \frac{1}{y_1}-\frac{1}{y_2}\right\vert + \frac{1}{y_2}\vert x_1-x_2\vert \\[1em]
                &< x_1\frac{4\delta}{y_1^2} + \epsilon\left(\frac{1}{y_1} + \frac{4\delta}{y_1^2}\right) \\[1em]
                &= \frac{x_1}{y_1}\left[\frac{4\delta}{y_1}+\epsilon\left(1+\frac{4\delta}{y_1}\right)\right]. \qedhere
            \end{align*}
        \end{enumerate}
    \end{proof}

    We now present our first result on the robustness of MCC to the choice of weights.
    For a dataset with true classes $\bvec{t}$ and a given prediction $\bvec{c}$, we show that $\mcc$ values change only by at most a small amount when the weights of the observations are changed slightly.

    \begin{theorem}
        \label{thm:mccstable}
        Let $S$ and $W$ be two weight matrices such that $\max \left|S_{ii}-W_{ii}\right|<\epsilon$,
        and let $\mcc_S$ and $\mcc_W$ be the $\mcc$ values corresponding to these weights.
        For non-zero binary $N$-vectors $\bvec{t}$ and $\bvec{c}$, neither equal to $\bvec{1}$ or $\bvec{0}$, and $A = \{ \langle\bvec{t},\bvec{c}\rangle_S, \langle\bvec{t},\bvec{1}\rangle_S,\langle\bvec{1},\bvec{c}\rangle_S,\langle\bvec{1}-\bvec{t},\bvec{1}\rangle_S,\langle\bvec{1},\bvec{1}-\bvec{c}\rangle_S \}$, let $M = \max A$ and $m = \min A$. 
        
        If $m>0$ and $0<\epsilon< \min\{m/2,1/N\}$, then 
        \begin{align*}
            \bigl|\,\mcc_S(\bvec{t},\bvec{c}) - & \mcc_W(\bvec{t},\bvec{c})\,\bigr| < \mcc_S(\bvec{t},\bvec{s})\cdot \frac{2^5 \tr(S)^2 N}{m^2} \left[ 1 + M^2\left( 1+\frac{2^5\tr(S)^2N}{m^2}\epsilon \right)\right] \cdot\epsilon
        \end{align*}
    \end{theorem}

    This bound guarantees that if the weights change by at most $\epsilon$, the $\mcc$ changes by at most a factor of $\epsilon$.
    Hence, one can trust the evaluation of the performance of a predictor using $\mcc$ even when the weights may be altered by small amounts.
    
    \begin{proof}
        We use the subscripts $S$ and $W$ to indicate the weights being used.
        By the triangle inequality,
        \[\bigl\vert \langle\bvec{a},\bvec{b}\rangle_S - \langle\bvec{a} ,\bvec{b}\rangle_W \bigr\vert < N\epsilon,\]
        for any binary vectors $\bvec{a}$ and $\bvec{b}$.  For the numerator, also observe that 
        \[\langle\bvec{t},\bvec{c}\rangle_S \leq M \leq \langle\bvec{1},\bvec{1}\rangle_S = \tr(S).\]
        Then, by Lemma \ref{lem:bdops},
        \begin{align*}
            \bigl\vert \langle\bvec{t},\bvec{c}\rangle_S \langle\bvec{1},\bvec{1}\rangle_S - \langle\bvec{t},\bvec{c}\rangle_W \langle\bvec{1},\bvec{1}\rangle_W \bigr\vert &< 2^2\tr(S)^2 N\epsilon, \\[1ex]
            \bigl\vert \langle\bvec{t},\bvec{1}\rangle_S \langle\bvec{1},\bvec{c}\rangle_S - \langle\bvec{t},\bvec{1}\rangle_W \langle\bvec{1},\bvec{c}\rangle_W \bigr\vert &< 2^2\tr(S)^2 N\epsilon.
        \end{align*}
        The triangle inequality then implies a bound for the change in the numerator is $2\bigl( 2^2\tr(S)^2 N\epsilon\bigr)$.

        For the denominator, let $a$, $b$, $c$, $d$ represent $\langle\bvec{t},\bvec{1}\rangle$,  $\langle\bvec{1},\bvec{c}\rangle$, $\langle\bvec{1}-\bvec{t},\bvec{1}\rangle$,  $\langle\bvec{1},\bvec{1}-\bvec{c}\rangle$, respectively.  By Lemma \ref{lem:bdops} and the inequality above,
        \[\left\vert (abcd)_S - (abcd)_W\right\vert < 2^4M^4N\epsilon.\]
        Then, by Lemma \ref{lem:bdops} again, 
        \begin{align*}
            \left\vert \sqrt{(abcd)_S} - \sqrt{(abcd)_W} \,\right\vert &<  \frac{2^4M^4N\epsilon}{\sqrt{2m^4}}.
        \end{align*}

        Applying \ref{lem:bdops}(e) with constants $\delta=2^3\tr(S)^2 N\epsilon$ and $\eta= \frac{2^4M^4N}{\sqrt{2}m^2}\epsilon$,
        \begin{align*}
            \bigl|\,\mcc_S(\bvec{t},\bvec{c}) - \mcc_W(\bvec{t},\bvec{c})\,\bigr| &< \mcc_S(\bvec{t},\bvec{s})\left[ \frac{4\delta}{m^2}+\eta\left(1+\frac{4\delta}{m^2}\right) \right] \\[1ex]
            &< \mcc_S(\bvec{t},\bvec{s})\cdot \frac{2^5 \tr(S)^2 N}{m^2} \left[ 1 + M^2\left( 1+\frac{4\delta}{m^2} \right)\right] \cdot\epsilon \\[1ex]
            &< \mcc_S(\bvec{t},\bvec{s})\cdot \frac{2^5 \tr(S)^2 N}{m^2} \left[ 1 + M^2\left( 1+\frac{2^5\tr(S)^2N}{m^2}\epsilon \right)\right] \cdot\epsilon
        \end{align*}
    \end{proof}

\section{Weighted Measures for the Multiclass Case}
    \label{mult_class}

    We study the sensitivity of measures ECC, $\mpc_1$, and $\mpc_2$ in Eqns.~(\ref{ecc})--(\ref{mpc2}) to small changes in weights.
    We use $\Vert\cdot \Vert$ to represent a matrix norm.

    \begin{prop}
        \label{prp:RS-RW}
        Let $S$ and $W$ be two weight matrices such that $\max |S_{ii}-W_{ii}|<\epsilon$.
 Let $s=\sum S_{ii}$ and let $0<\epsilon<s/2$.
 For binary $N$-vectors $\bvec{t}$ and $\bvec{c}$,
        \[\left\Vert R_{tc}^S - R_{tc}^W\right\Vert \leq \frac{4N\epsilon}{s^2}\sum_{n=1}^N \epsilon\,\bigl\Vert(\bvec{t}_n-\bar{\bvec{t}})(\bvec{c}_n-\bar{\bvec{c}})^{\top}\bigr\Vert\,.\]
    \end{prop}
    
    \begin{proof}
        Apply Lemma \ref{lem:bdops} with $s$ in place of $x$ and the observation that 
        \[\left|\sum S_{ii} - \sum W_{ii}\right| \leq \sum|S_{ii}-W_{ii}|\leq N\epsilon. \qedhere \] 
    \end{proof}

    \begin{rmk}
        \label{rmk:RS-RWexpr}
        Using matrix notation from Eqn.~(\ref{eq:Rtc}), the inequality takes the following form:
        \[\Vert R_{tc}^S - R_{tc}^W\Vert \leq \frac{4N\epsilon^2}{s^2}\Vert \bvec{t}-\bar{\bvec{t}}\bvec{e}^{\top}\Vert\,\Vert I\Vert\,\Vert (\bvec{c}-\bar{\bvec{c}}\bvec{e}^{\top})^{\top} \Vert.\]
    \end{rmk}

    It is observed by Stoica and Babu \cite{StBa2024} that 
    \begin{equation}
        \label{Rtckk}
        \bigl[R_{tc}\bigr]_{kk} = \frac{1}{\sum S_i} \sum_{n=1}^N S_n\bigl[t_n(k)-\bar{t}_k\bigr] \bigl[c_n(k)-\bar{c}_k\bigr].
    \end{equation}

    \begin{lemma}
        \label{lem:RtcS-RtcW}
        With same conditions as in Proposition \ref{prp:RS-RW},
        \begin{align*}
            \bigl\vert [R_{tc}^S]_{kk} & - [R_{tc}^W]_{kk}\bigr\vert\ \leq
            \frac{4N\epsilon^2}{s^2}\max_n\{|t_n(k)-\bar{t}_k|\}\,\max_n\{|c_n(k)-\bar{c}_k|\}.
        \end{align*}
    \end{lemma}
    \begin{proof}
        As in the proof of Proposition \ref{prp:RS-RW}, apply Lemma \ref{lem:bdops} and the observation that
        \[\left|\sum S_{ii} - \sum W_{ii}\right| \leq \sum|S_{ii}-W_{ii}|\leq N\epsilon. \qedhere \]
    \end{proof}

    \begin{lemma}
        \label{lem:trRtcS-trRtcW}
        Let $S$ and $W$ be weight matrices and $s=\tr(S)$.
        Suppose that $\max|S_{ii}-W_{ii}|<\epsilon < s/2$.
        Given the binary vectors $\bvec{t}$ and $\bvec{c}$, let $M_t = \max { \left| [\bvec{t}-\bar{\bvec{t}}\bvec{e}^{\top}]_{ij} \right| }$ and $M_c = \max { \left| [\bvec{c}-\bar{\bvec{c}}\bvec{e}^{\top}]_{ij} \right| }$.
        Then 
        \[\left| \tr(R_{tc}^S) - \tr(R_{tc}^W) \right| \leq \frac{4N^2\epsilon^2}{s^2} M_t M_c.\]
    \end{lemma}
    
    \begin{proof}
        Apply triangle inequality and Lemma \ref{lem:RtcS-RtcW}. \qedhere
    \end{proof}

    \begin{lemma}
        \label{lem:RSWbds}
        Consider the conditions as in Lemma \ref{lem:trRtcS-trRtcW}.
        \begin{enumerate}
            \item There is a constant $C$ depending on $\bvec{t}$, $\bvec{c}$ such that \label{lem:RSWbds:tr}
              \begin{equation}
                \label{eq:trRS-trRW}
                \bigl|\tr(R_{tt}^S)\, \tr(R_{cc}^S) - \tr(R_{tt}^W)\, \tr(R_{cc}^W) \bigr| < \frac{4N^2\epsilon^2}{s^2}\,C.
              \end{equation}

            \item \label{lem:RSWbds:Ck} There is a constant $C_k$ depending on $\bvec{t}$, $\bvec{c}$ such that 
            \[\bigl|[R_{tt}^S]_{kk}[R_{cc}^S]_{kk} - [R_{tt}^W]_{kk}[R_{cc}^W]_{kk}\bigr| \leq \frac{4N\epsilon^2}{s^2}\,C_k.\] 

            \item There is a constant $C$ depending on $\bvec{t}$, $\bvec{c}$ such that  \label{lem:RSWbds:sqrtS2W2}
            \begin{align*}
               \left| \sum_{k=1}^K\biggl(\bigl[R_{tt}^S\bigr]_{kk} \bigl[R_{cc}^S\bigr]_{kk} \biggr)^{\!\!1/2} - 
               \sum_{k=1}^K\biggl(\bigl[R_{tt}^W\bigr]_{kk} 
               \bigl[R_{cc}^W\bigr]_{kk}\biggr)^{\!\!1/2} \right| 
               \leq \frac{4N\epsilon^2}{s^2}\,C. 
            \end{align*}

            \item With an appropriately chosen constant $C$, \label{lem:RSWbds:frc}
            \begin{align*}
                \left| \frac{[R_{tc}^S]_{kk}}{\bigl([R_{tt}^S]_{kk}[R_{cc}^S]_{kk}\bigr)^{1/2}} - \frac{[R_{tc}^W]_{kk}}{\bigl([R_{tt}^W]_{kk}[R_{cc}^W]_{kk}\bigr)^{1/2}} \right| 
                \leq \hspace*{2.5in} \\
                \left\vert\frac{[R_{tc}^S]_{kk}}{\bigl([R_{tt}^S]_{kk}[R_{cc}^S]_{kk}\bigr)^{1/2}}\right\vert\,  
                \left[ \frac{4\cdot 4N C}{s^2}  + \frac{4N C}{s^2}\left(1+ \frac{4\cdot 4N C}{s^2} \!\cdot\! \epsilon^2\right)
                \right]\! \cdot \epsilon^2.\\
            \end{align*}
        \end{enumerate}
    \end{lemma}
    
    \begin{proof}
        \begin{enumerate}
            \item By Lemma \ref{lem:trRtcS-trRtcW},
            \begin{align*}
                \tr(R_{tt}^S) - \frac{4N^2\epsilon^2 M_t^2}{s^2} &< \tr(R_{tt}^W)
                < \tr(R_{tt}^S) + \frac{4N^2\epsilon^2 M_t^2}{s^2},
            \end{align*}
            and 
            \begin{align*}
                \tr(R_{cc}^S) - \frac{4N^2\epsilon^2 M_c^2}{s^2} < \tr(R_{cc}^W) 
                < \tr(R_{cc}^S) + \frac{4N^2\epsilon^2 M_c^2}{s^2}.
            \end{align*}
            Let $C = M_t^2\tr(R_{cc}^S) + M_c^2\tr(R_{tt}^S) + \frac{4N^2}{s^2}$.  Multiplying the above two inequalities and recalling $\epsilon<1$ imply  
            \begin{align*}
                \tr(R_{tt}^S)\tr(R_{cc}^S)  -  \frac{4N^2\epsilon^2}{s^2}\,C < \tr(R_{tt}^W)\tr(R_{cc}^W) 
                < \tr(R_{tt}^S)\tr(R_{cc}^S) + \frac{4N^2\epsilon^2}{s^2}\,C.
            \end{align*}
            
            \item By Lemma \ref{lem:RtcS-RtcW}, 
            \begin{align*}
                [R_{tt}^S]_{kk} - \frac{4N\epsilon^2}{s^2}\,M_t^2 \leq [R_{tt}^W]_{kk} 
                \leq [R_{tt}^S]_{kk} + \frac{4N\epsilon^2}{s^2}\,M_t^2,
            \end{align*}
            and 
            \begin{align*}
                [R_{cc}^S]_{kk} - \frac{4N\epsilon^2}{s^2}\,M_c^2 &\leq [R_{cc}^W]_{kk} 
                \leq [R_{cc}^S]_{kk} + \frac{4N\epsilon^2}{s^2}\,M_c^2.
            \end{align*}
            Let $C_k = [R_{tt}^S]_{kk}M_c^2 + [R_{cc}^S]_{kk}M_t^2 + \frac{4N}{s^2}$.
 Multiplying both sides of the above equations implies 
            \begin{align*}
                [R_{tt}^S]_{kk}[R_{cc}^S]_{kk} - &\frac{4N\epsilon^2}{s^2}\,C_k \leq [R_{tt}^W]_{kk}[R_{cc}^W]_{kk} 
                \leq [R_{tt}^S]_{kk}[R_{cc}^S]_{kk} + \frac{4N\epsilon^2}{s^2}\,C_k.
            \end{align*} 

            \item By proof of Part \ref{lem:RSWbds:Ck} above and Lemma \ref{lem:bdops}, 
            \begin{align*}
                \left|\biggl([R_{tt}^S]_{kk}[R_{cc}^S]_{kk}\biggr)^{\!\!1/2} \! - \! \right. & \left. \biggl([R_{tt}^W]_{kk}[R_{cc}^W]_{kk}\biggr)^{\!\!1/2}\right| \leq 
                \frac{4N\epsilon^2}{s^2}\cdot \frac{C_k}{\sqrt{2[R_{tt}^S]_{kk}[R_{cc}^S]_{kk}}}.
            \end{align*}
            Apply the triangle inequality and take 
            \[C = \sum_{k=1}^K\frac{C_k}{\sqrt{2[R_{tt}^S]_{kk}[R_{cc}^S]_{kk}}}.\]
            
            \item Select $C$ to be an upper bound for the bounds needed in Lemma \ref{lem:RtcS-RtcW} and the proofs in the items above, and apply Lemma \ref{lem:RtcS-RtcW} and Lemma \ref{lem:bdops}(e).
        \end{enumerate}
    \end{proof}

    We use the above results to now present the main result certifying the robustness of the three multiclass measures of performance to small changes in the weights.
    These results correspond to the robustness bound in Theorem \ref{thm:mccstable} for the $\mcc$ in the binary case.
    Just as in that case, we get here as well that each measure changes by a factor of at most $\epsilon^2$ when the weights are changed by $\epsilon$.

    \begin{theorem}
        \label{thm:msrstable}
        Consider the same notation and conditions as in Lemma \ref{lem:trRtcS-trRtcW}.
        We have that each constant $C$ below depends on $\bvec{t}$ and $\bvec{c}$ to give the following results.
        \begin{enumerate}
            \item For $y=\tr(R_{tt}^S)\tr(R_{cc}^S)$, there is a constant $C$ such that
                \begin{equation}
                    \label{eq4.3.2}
                    \bigl|R_K^S - R_K^W\bigr| \leq \vert R_K^S\vert \left[\frac{4\cdot 4N^2 C}{y\sqrt{2}s^2} + \frac{4N^2M_tM_c}{s^2}\left(1+\frac{4\cdot 4NC}{y\sqrt{2}s^2}\epsilon^2 \right)\right] \cdot \epsilon^2.
                \end{equation}
                
            \item For $y = \sum\bigl([R_{tt}^S]_{kk}[R_{cc}^S]_{kk}\bigr)^{1/2}$, there is a constant $C$ such that 
            \begin{equation*}
                \left|\mpc_1^S - \mpc_1^W\right| \leq \vert\mpc_1^S\vert \left[ 
                \frac{4\cdot 4NC}{ys^2} + \frac{4N^2M_tM_c}{s^2}\left(1 + \frac{4\cdot 4NC}{ys^2}\epsilon^2\right)
                \right] \cdot \epsilon^2.
            \end{equation*}

            \item There is a constant $C$ such that 
            \[\hspace*{-0.15in} \left|\mpc_2^S - \mpc_2^W\right| \leq \sum\left\vert \frac{[R_{tc}^S]_{kk}}{K\bigl([R_{tt}^S]_{kk}[R_{cc}^S]_{kk}\bigr)^{1/2}} \right\vert \left[ \frac{4\cdot 4NC}{s^2} + \frac{4NC}{s^2}\left(1 + \frac{4\cdot 4NC}{s^2}\epsilon^2\right) \right]\, \cdot \epsilon^2.\]
        \end{enumerate} 
    \end{theorem}
    
    \begin{proof}
        In each case, make appropriate use of Lemmas \ref{lem:bdops}, \ref{lem:RtcS-RtcW}--\ref{lem:RSWbds}.
    \end{proof}

\section{Computational Results}
\label{computations}

We present a detailed computational study highlighting the effects of individual weights of observations on the performance of $\mcc$ in the binary case as well as on the weighted measures for the multiclass case.

Figure \ref{fig:sep_plots} presents the results for the binary case.
We consider samples of size $N=150$.
We generate the vector $\bvec{t}$ of true classes values, and then generate a collection of prediction vectors $\bvec{c}$ whose classification performance is to be measured by the weighted $\mcc$.
A contiguous section of $\bvec{c}$ of length $N/3=50$ is created to match the corresponding section of $\bvec{t}$ to a given proportion $p$.
We consider three values for this match proportion: $p=\{0, 0.5, 1.0\}$.
The remainder of $\bvec{c}$ will match the corresponding section of $\bvec{t}$ with proportion $p_0=0.5$.
The horizontal axis for the plots indicates the initial index in $\bvec{c}$ of the section matching with proportion $p$.
For instance, if this initial index is $51$, then observations 51--100 in $\bvec{c}$ are matched exactly with the corresponding ones in $\bvec{t}$.
The vertical axis records the average $\mcc$ value over $100$ samples.

For comparison with the default (unweighted) MCC, we report the averages labeled as $\mcc$ when all observations have the same weight (we set each diagonal entry of the weight matrix $S$ as $1$).
The weighted case is reported under $\wmcc$, for which the diagonal entries of $S$ are set such that the first third of the weights are $1$, the middle third are $100$, and the last third are $10000$.

\begin{figure}[ht!]
  \centering
  \includegraphics[width=0.7\textwidth]{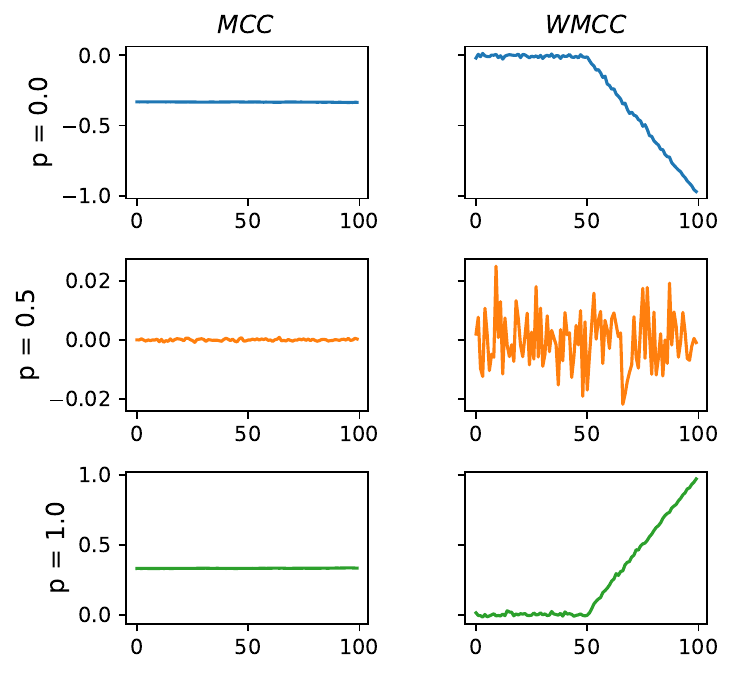}
  \caption{Plots of the mean $\mcc$ and $W\mcc$ for 100 sample vectors $\bvec{c}$ with 1/3 of the sample matching at proportion $p$ and the other 2/3 matching at $p_0=0.5$.
    The section matching at $p$ is contiguous and starts at the index indicated along the horizontal axis.}
  \label{fig:sep_plots}
\end{figure}

The graphs show that the last third of the sample, the portion with weights set to $10000$, drives the value of $\wmcc$.
When the matching portion $p=0.5$ for the entire sample vector $\bvec{c}$, a value of $0$ is expected for both $\mcc$ and $\wmcc$.
The imbalance in weights causes higher variation, but this is evident in the middle row of Figure \ref{fig:sep_plots}.
As $p$ varies, the effect of the change in matching proportion in the higher weighted entries is evident in the $\wmcc$ column: each graph lingers close to $0$ until the starting index of the section with proportion $p$ passes $50$---when at least half of this segment lies in the portion of $\bvec{c}$ with the highest weights.
Similarly, for $p=0$,
$\wmcc$ ranks worse (values closer to $-1$) the predictions that misclassify the highly weighted observations.

A similar experiment is presented for the multiclass case in Figure \ref{fig:mult_sep_plots}.
Here, the values of $\ecc$, $\wecc$, $\mpc_1$, $\wmpc_1$, $\mpc_2$, and $\wmpc_2$ are calculated for $K\times N=3\times 150$ matrices $\bvec{t}$ and $\bvec{c}$, representing values for $K=3$ classes and $N=150$ cases.
As for the binary case, the columns of the matrices are weighted such that the first 50 have weight 1, the second 50 have weight 100, and the last 50 have weight 10000.
A contiguous section of columns from $\bvec{c}$ of length $N/3$ are set to match the corresponding columns of $\bvec{t}$ at a given proportion $p$.
The remaining columns of $\bvec{c}$ match the corresponding columns of $\bvec{t}$ at proportion $p_0=0.5$.
The graphs show the average of each metric over 100 different matrices $\bvec{c}$.
As in the binary case, the effect of the weights is evident by the change in the graphs at index value $50$, as more than half of the section matching at proportion $p$ has heavier weights.
It should be noted that unlike the binary case, the unweighted values are non-negative.
Also, for the given samples the values of $\wecc$, $\wmpc_1$, and $\wmpc_2$ are almost identical.

\section{Discussion}
\label{sec_disc}

Our primary objectives have been to provide an analysis of the effects of weights on individual observations on the values of $\mcc$ and its extensions to multiclass case, and then to compare the behavior of unweighted metrics to the corresponding weighted metrics.
With regard to the former, there is a rigorous relationship between changes in weights and changes in the values for the metrics---an adjustment in the weights of up to $\epsilon$ leads to changes in values by a factor of at most $\epsilon^2$.
At the same time, the exact nature of this relationship may be complicated by imbalances in the data and warrant further investigation.
Such imbalances were addressed by Radoux and Bogaert \cite{RadouxBogaert} through enhanced versions of the metrics that eliminate common zeros.
Perhaps a similar improvement can be accomplished by applying weights to observations.
Another source of variation can occur at possible discontinuities for the metrics, as mentioned in Remark \ref{rmk_0den}.  

Several metrics exist, and a critical part of classification is the selection of an appropriate metric \cite{Reetaletal2024}.
The use of distinct weights for individual observations may help address some concerns for metrics, and doing so only modestly increases the computational burden.

\begin{landscape}
  \vspace*{1in}
  \begin{figure}[ht!]
    \hspace*{-0.1in}
    \includegraphics[width=9in]{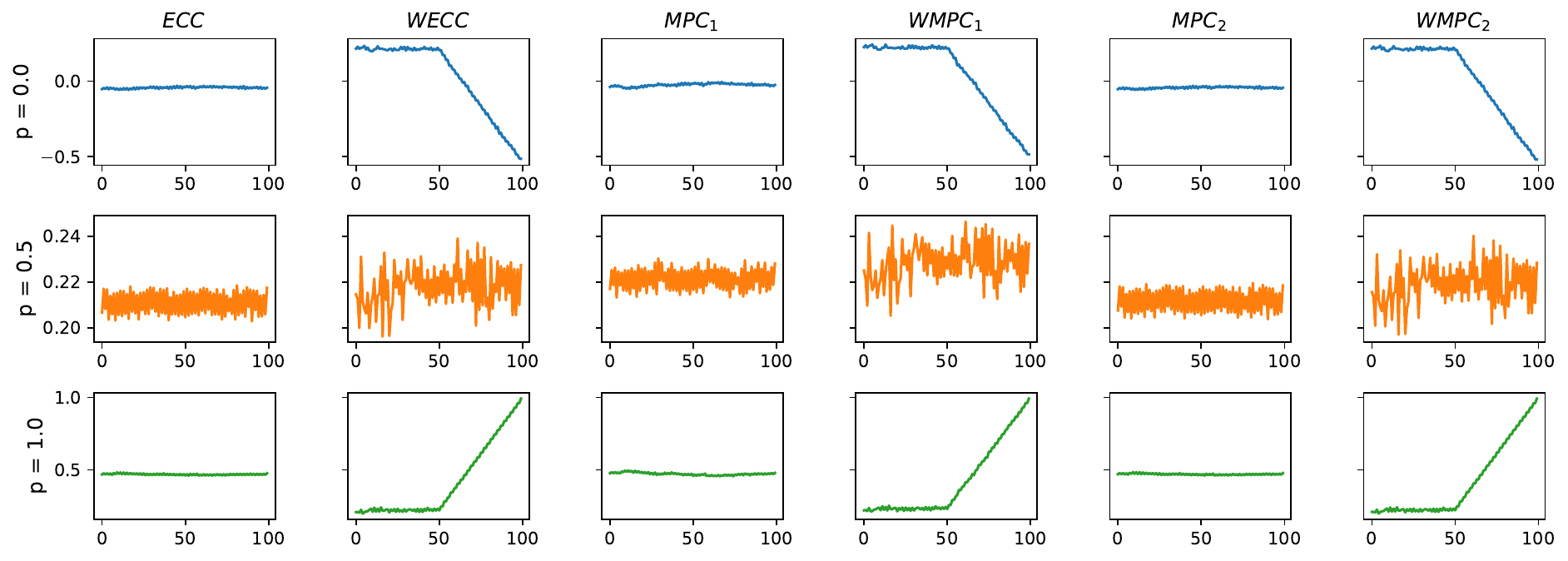}
    \caption{Plots of the mean values of each multiclass metric for 100 sample matrices $\bvec{c}$ with 1/3 of the sample matching at proportion $p$ and the other 2/3 matching at $p_0=0.5$.
      The section matching at $p$ is contiguous and starts at the index indicated along the horizontal axis.}
    \label{fig:mult_sep_plots}
  \end{figure}
\end{landscape}



\end{document}